\documentclass{article}

\usepackage{microtype}
\usepackage{graphicx}
\usepackage{booktabs} 

\usepackage{hyperref}

\usepackage{amsfonts}       
\usepackage{url}
\usepackage{amsmath}
\usepackage{amsthm}
\usepackage{wrapfig}
\usepackage{mathtools}
\usepackage{tikz}
\usepackage{subfig}
\usepackage{algorithmic}
\usepackage{subfig}
\usepackage{amssymb}
\usepackage{color}

\newcounter{thm_counter}
\newcounter{lem_counter}
\newcounter{pro_counter}

\newcounter{ass_counter}

\newcounter{rmk_counter}

\newtheorem{theorem}[thm_counter]{Theorem}
\newtheorem{proposition}[pro_counter]{Proposition}
\newtheorem{lemma}[lem_counter]{Lemma}

\newtheorem{assumption}[ass_counter]{Assumption}

\newtheorem{remark}[rmk_counter]{Remark}

\newcommand{\mop}{{\mathcal{T}}}

\newcommand{\eopt}{\epsilon_{\text{opt}}}

\usepackage{mathtools}
\usepackage{mathrsfs}
\mathtoolsset{showonlyrefs}
\newcommand{\E}{\mathbb{E}}
\newcommand{\R}{\mathbb{R}}

\usepackage[inline,ignoremode]{trackchanges}
\addeditor{lb}

\addeditor{sz}

\addeditor{yh}

\usepackage[accepted]{icml2020}

\icmltitlerunning{GradientDICE}

\begin{document}
\twocolumn[
\icmltitle{GradientDICE: Rethinking Generalized Offline Estimation\\of Stationary Values}



\icmlsetsymbol{equal}{*}

\begin{icmlauthorlist}
\icmlauthor{Shangtong Zhang}{ox}
\icmlauthor{Bo Liu}{au}
\icmlauthor{Shimon Whiteson}{ox}
\end{icmlauthorlist}

\icmlaffiliation{ox}{University of Oxford}
\icmlaffiliation{au}{Auburn University}

\icmlcorrespondingauthor{Shangtong Zhang}{\mbox{shangtong.zhang@cs.ox.ac.uk}}

\icmlkeywords{Reinforcement Learning, Off-policy Evaluation}

\vskip 0.3in
]



\printAffiliationsAndNotice{}  

\begin{abstract}
We present GradientDICE for estimating the density ratio between the state distribution of the target policy and the sampling distribution in off-policy reinforcement learning.
GradientDICE fixes several problems of GenDICE \citep{zhang2020gendice}, the state-of-the-art for estimating such density ratios. 
Namely, the optimization problem in GenDICE is \emph{not} a convex-concave saddle-point problem once nonlinearity in optimization variable parameterization is introduced to ensure positivity, 
so any primal-dual algorithm is \emph{not} guaranteed to converge or find the desired solution. 
However, such nonlinearity is essential to ensure the consistency of GenDICE even with a tabular representation.
This is a fundamental contradiction,
resulting from GenDICE's original formulation of the optimization problem.
In GradientDICE, we optimize a different objective from GenDICE
by using the Perron-Frobenius theorem and eliminating GenDICE's use of divergence,
such that nonlinearity in parameterization is not necessary for GradientDICE, 
which is provably convergent under linear function approximation.
\end{abstract}
\setcounter{footnote}{2}
\section{Introduction}
A key challenge in reinforcement learning (RL, \citealt{sutton2018reinforcement}) is off-policy evaluation \citep{precup2001off,maei2011gradient,jiang2015doubly,sutton2018reinforcement,liu2018breaking,nachum2019dualdice,zhang2020gendice},
where we want to estimate the performance of a target policy (average reward in the continuing setting or expected total discounted reward in the episodic setting \citep{puterman2014markov}),
from data generated by one or more behavior policies.
Compared with on-policy evaluation \citep{sutton1988learning},
which requires data generated by the target policy,
off-policy evaluation is more flexible.
We can evaluate a new policy with existing data in a replay buffer \citep{lin1992self} without interacting with the environment again.
We can also evaluate multiple target policies simultaneously when following a single behavior policy \citep{sutton2011horde}.

One major challenge in off-policy evaluation is dealing with distribution mismatch: the state distribution of the target policy is different from the sampling distribution.
This mismatch leads to divergence of the off-policy linear temporal difference learning algorithm \citep{baird1995residual,tsitsiklis1997analysis}.
\citet{precup2001off} 
address this issue with products of importance sampling ratios,
which, however, suffer from a large variance.
To correct for distribution mismatch without incurring a  large variance,
\citet{hallak2017consistent,liu2018breaking} propose to directly learn the density ratio between the state distribution of the target policy and the sampling distribution using function approximation.\footnote{Such density ratios are referred to as \emph{stationary values} in \citet{zhang2020gendice}}
Intuitively, this learned density ratio is a ``marginalization'' of the products of importance sampling ratios.

The density ratio learning algorithms from \citet{hallak2017consistent,liu2018breaking} require data generated by a \emph{single known} behavior policy.
\citet{nachum2019dualdice} relax this constraint in DualDICE, 
which is compatible with \emph{multiple unknown} behavior policies and \emph{offline} training. 
DualDICE, however, copes well with only the total discounted reward criterion and cannot be used under the average reward criterion. 
In particular, DualDICE becomes unstable as the discounting factor grows towards 1 \citep{zhang2020gendice}.
To address the limitation of DualDICE, \citet{zhang2020gendice} propose GenDICE, 
which is compatible with multiple unknown behavior policies and offline training \emph{under both criteria}. 
\citet{zhang2020gendice} show empirically that GenDICE achieves a new state-of-the-art in off-policy evaluation.

In this paper, we point out key problems with GenDICE. 
In particular, the optimization problem in GenDICE is \emph{not} a convex-concave saddle-point problem (CCSP) once nonlinearity in optimization variable parameterization is introduced to ensure positivity, 
so any primal-dual algorithm is not guaranteed to converge or find the desired solution even with tabular representation. 
However, such positivity is essential to ensure the consistency of GenDICE.
This is a fundamental contradiction,
resulting from GenDICE's original formulation of the optimization problem.

Furthermore, we propose GradientDICE, which overcomes these problems.  GradientDICE optimizes a different objective from GenDICE
by using the Perron-Frobenius theorem \citep{horn2012matrix} and eliminating GenDICE's use of divergence.
Consequently, nonlinearity in parameterization is not necessary for GradientDICE, which is provably convergent under linear function approximation.
Finally, we provide empirical results demonstrating the advantages of GradientDICE over GenDICE and DualDICE. 

\section{Background}
We use vectors and functions interchangeably when this does not confuse.
For example, let $f: \mathcal{X} \times \mathcal{Y} \rightarrow \R$ be a function;
we also use $f$ to denote the corresponding vector in $\R^{|\mathcal{X}| \times |\mathcal{Y}|}$.
All vectors are column vectors.

We consider an infinite-horizon MDP with a finite state space $\mathcal{S}$, a finite action space $\mathcal{A}$, 
a transition kernel $p: \mathcal{S} \times \mathcal{S} \times \mathcal{A} \rightarrow [0, 1]$, a reward function $r: \mathcal{S} \times \mathcal{A} \rightarrow \R$,
a discount factor $\gamma \in [0, 1)$, 
and an initial state distribution $\tilde{\mu}_0$.
The initial state $S_0$ is sampled from $\tilde{\mu}_0$.
At time step $t$, an agent at $S_t$ selects an action $A_t$ according to $\pi(\cdot | S_t)$, where $\pi: \mathcal{A} \times \mathcal{S} \rightarrow [0, 1]$ is the policy being followed by the agent. 
The agent then proceeds to the next state $S_{t+1}$ according to $p(\cdot | S_t, A_t)$ and gets a reward $R_{t+1}$ satisfying $\mathbb{E}[R_{t+1}] = r(S_t, A_t)$.

Similar to \citet{zhang2020gendice}, we consider two performance metrics for the policy $\pi$:
the total discounted reward 
$\rho_\gamma(\pi) \doteq (1 - \gamma) \E[\sum_{t=0}^\infty \gamma^t R_{t+1}]$
and the average reward
$\rho(\pi) \doteq \lim_{T \rightarrow \infty} \frac{1}{T} \E[\sum_{t=1}^T R_t]$.
When the Markov chain induced by $\pi$ is ergodic, 
$\rho(\pi)$ is always well defined \citep{puterman2014markov}.
Throughout this paper, we implicitly assume this ergodicity whenever we consider $\rho(\pi)$.
When considering $\rho_\gamma(\pi)$, we are interested in the normalized discounted state-action occupation $d_\gamma(s, a)$.
Let $d_t^\pi(s, a) \doteq \Pr(S_t = s, A_t = a | \tilde{\mu}_0, p, \pi)$ be the probability of occupying the state-action pair $(s, a)$ at the time step $t$ following $\pi$. Then,
we have 
$d_{\gamma}(s, a) \doteq (1 - \gamma) \sum_{t=0}^\infty \gamma^t d_t^\pi(s, a)$.
When considering $\rho(\pi)$, we are interested in the stationary state-action distribution 
$d_{\gamma}(s, a) \doteq \lim_{t\rightarrow \infty} d_t^\pi(s, a)$.
To simplify notation, we extend the definition of $\rho_\gamma(\pi)$ and $d_\gamma(s, a)$ from $\gamma \in [0, 1)$ to $\gamma \in [0, 1]$ by defining $\rho_1(\pi) \doteq \rho(\pi)$ and $d_1(s, a) \doteq d(s, a)$.
It follows that for any $\gamma \in [0, 1]$, we have
$\rho_\gamma(\pi) = \E_{(s, a)\sim d_\gamma}[r(s, a)]$.

We are interested in estimating $\rho_\gamma(\pi)$ without executing the policy $\pi$.
Similar to \citet{zhang2020gendice}, we assume access to a fixed dataset
$\mathcal{D} \doteq \{ (s_i, a_i, r_i, s_i^\prime) \}_{i=1}^N$.
Here the state-action pair $(s_i, a_i)$ is sampled from an unknown distribution $d_\mu: \mathcal{S} \times \mathcal{A} \rightarrow [0, 1]$,
which may result from multiple unknown behavior policies.
The reward $r_i$ satisfies $\E[r_i] = r(s_i, a_i)$.
The successor state $s_i^\prime$ is sampled from $p(\cdot | s_i, a_i)$.
As $\rho_\gamma(\pi) = \E_{(s, a) \sim d_\mu}[\frac{d_\gamma(s,a)}{d_\mu(s, a)}r(s, a)]$,
one possible approach for estimating $\rho_\gamma(\pi)$ is to learn the density ratio $\tau_*(s, a) \doteq \frac{d_\gamma(s,a)}{d_\mu(s, a)}$ directly.

We assume $d_\mu(s, a) > 0$ for every $(s, a)$ 
and use $D \in \R^{N_{sa} \times N_{sa}} (N_{sa} \doteq |\mathcal{S}| \times |\mathcal{A}|)$ to denote a diagonal matrix 
whose diagonal entry is $d_\mu$.
Let $\mu_0(s, a) \doteq \tilde{\mu}_0(s)\pi(a|s)$,
\citet{zhang2020gendice} show
\begin{align}
\label{eq:tau_recursion}
D\tau_* = \mop \tau_*, 
\end{align}
where the operator $\mop$ is defined as
\begin{align}
\mop y \doteq (1 - \gamma)\mu_0 + \gamma P^\top_\pi Dy,
\end{align}
and $P_\pi \in \R^{N_{sa} \times N_{sa}}$ is the state-action pair transition matrix, i.e.,
$P_\pi((s, a), (s^\prime, a^\prime)) \doteq p(s^\prime|s, a) \pi(a^\prime | s^\prime)$.
The operator $\mop$ is similar to the Bellman operator but in the reverse direction.
Similar ideas have been explored by \citet{wang2007dual,wang2008stable,hallak2017consistent,liu2018breaking,gelada2019off}.
As $D\tau_*$ is a probability measure, \citet{zhang2020gendice} propose to compute $\tau_*$ by solving the following optimization problem:
\begin{align}
\label{pb:original_problem}
\textstyle
\min_{\tau \in \R^{N_{sa}}, \tau \succeq 0} \, D_\phi(\mop \tau || D \tau) 
\quad \text{s.t.} \, d_\mu^\top \tau = 1,
\end{align}
where $\tau \succeq 0$ is elementwise greater or equal, $D_\phi$ is an $f$-divergence \citep{nowozin2016f} associated with a convex, lower semi-continuous generator function $\phi: \R_+ \rightarrow \R$ with $\phi(1) = 0$.
Let $q_1, q_2$ be two probability measures; we have
$D_\phi(q_1||q_2) \doteq \sum_x q_2(x) \phi( \frac{q_1(x)}{q_2(x)})$.
The $f$-divergence is used mainly for the ease of optimization but see \citet{zhang2020gendice} for discussion of other possible divergences.
Due to the difficulty in solving the constrained problem~\eqref{pb:original_problem} directly,
\citet{zhang2020gendice} propose to solve the following problem instead:
\begin{align}
\label{pb:penalty_problem}
\textstyle
\min_{\tau \in \R^{N_{sa}}, \tau \succeq 0} \, D_\phi(\mop \tau || D \tau) + \frac{\lambda}{2}(d_\mu^\top \tau - 1)^2,
\end{align}
where $\lambda > 0$ is a constant.
We have
\begin{lemma}
\label{lem:their_lemma} 
\citep{zhang2020gendice}
For any constant $\lambda > 0$, $\tau$ is optimal for the problem~\eqref{pb:penalty_problem} iff $\tau = \tau_*$.
\end{lemma}
To make the optimization tractable and address a double sampling issue, \citet{zhang2020gendice} rewrite $\phi(x)$ as $\phi(x) = \max_{f \in \R} xf - \frac{1}{2}\phi^*(f)$, where $\phi^*$ is the Fenchel conjugate of $\phi$, and use the interchangeability principle for interchanging maximization and expectation \citep{shapiro2014lectures,dai2016learning},
yielding the following problem, 
which is equivalent to~\eqref{pb:penalty_problem}:
\begin{align}
\label{pb:fake_ccsp}
\min_{\tau \in \R^{N_{sa}}, \tau \succeq 0} \quad \max_{f \in \R^{N_{sa}}, \eta \in \R} J(\tau, f, \eta),
\end{align}
where 
\label{eq:def-J}
\begin{align}
&J(\tau, f, \eta) \\
\doteq& (1 - \gamma) \E_{\mu_0}[f(s, a)] + \gamma \E_p[\tau(s, a) f(s^\prime, a^\prime)] - \\
&\E_{d_\mu}[\tau(s, a) \phi^*(f(s, a))] + \lambda \Big(\E_{d_\mu}[\eta \tau(s, a) - \eta] - \textstyle{\frac{\eta^2}{2}} \Big).
\end{align}
Here $\E_{\mu_0}, \E_{d_\mu}, \E_p$ are shorthand for $\E_{(s, a) \sim {\mu_0}}$, $\E_{(s, a) \sim d_\mu}$, $\E_{(s, a)\sim d_\mu, s^\prime \sim p(\cdot | s, a), a^\prime \sim \pi(\cdot | s^\prime)}$ respectively.
\citet{zhang2020gendice} show $J$ is convex in $\tau$ and concave in $\eta, f$, i.e., \eqref{pb:fake_ccsp} is a convex-concave saddle-point problem.
\citet{zhang2020gendice} therefore use a primal-dual algorithm (i.e, perform stochastic gradient ascent on $\eta, f$ and stochastic gradient descent on $\tau$) to find the saddle-point,
yielding GENeralized stationary DIstribution Correction Estimation (GenDICE). 

\section{Problems with GenDICE}

\subsection{Use of Divergences as the Objective}
\citet{zhang2020gendice} proposes to consider a family of divergences, the $f$-divergences.
However, $f$-divergences are defined between probability measures.
So $D_\phi$ in \eqref{pb:penalty_problem} implicitly requires its arguments to be valid probability measures. Consequently, \eqref{pb:penalty_problem}
 still has the implicit constraint that $d_\mu^\top \tau = 1$.
The main motivation for \citet{zhang2020gendice} to transform \eqref{pb:original_problem} into \eqref{pb:penalty_problem} is to get rid of this equality constraint. By using divergences, they do not really get rid of it.
When this implicit constraint is considered, the problem~\eqref{pb:penalty_problem} is still hard to optimize, as discussed in \citet{zhang2020gendice}.

We can, of course, just ignore this implicit constraint and interpret $D_\phi$ as a generic function instead of a divergence.
Namely, we do not require its arguments to be valid probability measures.
In this scenario, however, there is no guarantee that $D_\phi$ is always nonnegative, which plays a central role in proving Lemma~\ref{lem:their_lemma}'s claim that GenDICE is consistent.
Consider, for example, the KL-divergence, where $\phi(x) = x\log(x)$. If $q_2(x) > q_1(x) > 0$ holds for all $x$ (which is impossible when $q_1$ and $q_2$ are probability measures), 
clearly we have $D_\phi(q_1 || q_2) < 0$.
While \citet{zhang2020gendice} propose not to use KL divergence due to numerical instability, here we provide a more principled explanation that if KL divergence is used, Lemma~\ref{lem:their_lemma} does not necessarily hold.  
\citet{zhang2020gendice} propose to use $\chi^2$-divergence instead.
Fortunately, $\chi^2$-divergence has the property that $q_1(x) > 0 \wedge q_2(x) > 0$ implies $D_\phi(q_1 || q_2 ) \geq 0$,
even if $q_1$ and $q_2$ are not probability measures.
This property ensures Lemma~\ref{lem:their_lemma} holds even we just consider $D_\phi$ as a generic function instead of a divergence. 
But not all divergences have this property.
Moreover, even if $\chi^2$-divergence is considered, $\tau \succeq 0$ is still necessary for Lemma~\ref{lem:their_lemma} to hold.
This requirement ($\tau \succeq 0$) is also problematic, as discussed in the next section.

To summarize, we argue that \emph{$f$-divergence is not a good choice to form the optimization objective for density ratio learning}.

\subsection{Use of Primal-Dual Algorithms as the Solver}
We assume $\tau$, $f$ are parameterized by $\theta^{(1)}, \theta^{(2)}$.
As \eqref{pb:fake_ccsp} requires $\tau(s, a) \geq 0$,
\citet{zhang2020gendice} propose to add extra nonlinearity, e.g., $(\cdot)^2, \log(1 + \exp(\cdot))$, or $\exp(\cdot)$, in the parameterization of $\tau$.
Plugging the approximation in~\eqref{pb:fake_ccsp} yields
$\min_{\theta^{(1)}} \max_{\theta^{(2)}, \eta} J(\tau_{\theta^{(1)}}, \eta, f_{\theta^{(2)}})$.
Here $\tau_{\theta^{(1)}}$ and $f_{\theta^{(2)}}$ emphasize that $\tau$ and $f$ are parameterized functions.

There is now a contradiction.
On the one hand, $J(\tau_{\theta^{(1)}}, \eta, f_{\theta^{(2)}})$ is \emph{not} necessarily CCSP when nonlinearity is introduced in the parameterization.
In the definition of $J$ in \eqref{pb:fake_ccsp}, the sign of $\tau$ depends on $f$ and $\eta$. 
Unless $\tau$ is linear in $\theta^{(1)}$, the convexity of $J$ w.r.t.\ $\theta^{(1)}$ is in general hard to analyze \citep{boyd2004convex},
even we just add $(\cdot)^2$ after a linear parameterization.
Although \citet{zhang2020gendice} demonstrate great empirical success from a primal-dual algorithm,
this optimization procedure is \emph{not} theoretically justified as $J(\tau_{\theta^{(1)}}, \eta, f_{\theta^{(2)}})$ is not necessarily a convex-concave function. 
On the other hand, if we do not apply any nonlinearity in $\theta^{(1)}$, 
there is no guarantee that $\tau_{\theta^{(1)}}(s, a) > 0$ even with a tabular representation. 
Then Lemma~\ref{lem:their_lemma} does not necessarily hold, and GenDICE is not necessarily consistent.

To summarize, \emph{applying the primitive primal-dual algorithm for the GenDICE objective is not theoretically justified, even with a tabular representation}.

\subsection{Projection and Self-Normalization}

Besides applying nonlinearity, one may also consider projection to account for the constraint $\tau_{\theta^{(1)}}(s, a) > 0$, 
i.e., we project $\theta^{(1)}$ back to $\Theta \doteq \{\theta^{(1)} | \tau_{\theta^{(1)}}(s, a) > 0\}$. 
One direct instantiation of this idea is Projected Stochastic Gradient Descent (PSGD).
With nonlinear function approximation, it is not clear how to achieve this.
With tabular or linear representation, 
such projection reduces to inequality-constrained quadratic programming,
solving which usually involves subroutine numerical optimization,
making it very computationally expensive.
Moreover, 
the number of constraints grows linearly w.r.t. the number of states (not state features), 
indicating such projection does not scale well. 
Stochastic Mirror Descent (SMD, \citealt{mid:beck2003}) with
generalized KL-divergence is also a possible way to 
to achieve such a projection.
SMD, as well as PSGD, walks $\theta^{(1)}$ in the positive orthant.
To ensure $\tau_{\theta^{(1)}}(s, a) > 0$,
they also require positive features.
However, 
it is possible that the oracle $\tau$ lies in the feature space but the optimal $\theta^{(1)}$ does not lie in the positive orthant,
indicating SMD and PSGD can yield arbitrarily large optimization errors.

Self-normalization \citep{liu2018breaking,zhang2020gendice} is also an approach to ensure nonlinearity,
where we normalize the $\tau$ prediction over all possible state-action pairs.
Recently, \citet{mousavi2020blackbox} propose an objective with Maximum Mean Discrepancy and use self-normalization.
However, self-normalization usually generates biased solutions and is computationally prohibitive for large datasets (see Appendix E.3 in \citet{zhang2020gendice}).
Moreover, self-normalization in general yields non-convex objectives even with tabular or linear representation (e.g., the objective in \citet{mousavi2020blackbox} is non-convex),
rendering difficulties in optimization.

To summarize, we argue that \emph{neither projection nor self-normalization is a theoretically justified solution for the problems of GenDICE}.

\subsection{A Hard Example for GenDICE}
\begin{figure}[h]
\centering
\includegraphics[width=0.15\textwidth]{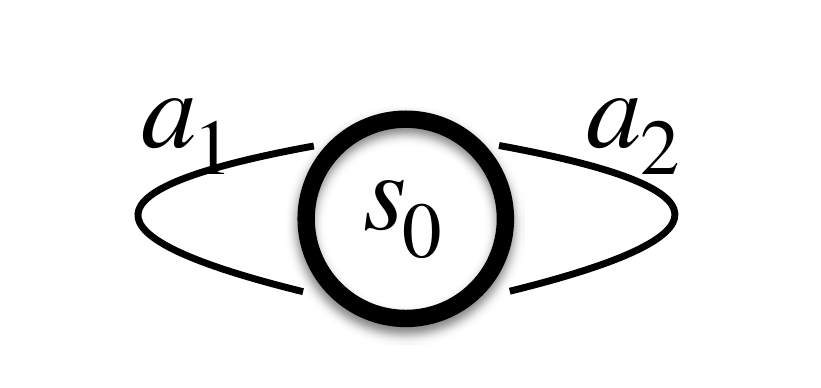}
\caption{\label{fig:single-state-MDP} A single-state MDP.}
\end{figure}

We now provide a concrete example to demonstrate the defects of GenDICE.
We consider a single state MDP (Figure~\ref{fig:single-state-MDP}) with two actions, 
both of which lead to that single state.
We set $\gamma = 1, d_\mu(s_0, a_1) = d_\mu(s_0, a_2) = \pi(a_1 | s_0) = \pi(a_2 | s_0) = 0.5$. 
Therefore, we have $\mu_0(s_0, a_1) = \mu_0(s_0, a_2) = 0.5$.
Under this setting, it is easy to verify that $\tau_* = [1, 1]^\top$.
We now instantiate \eqref{pb:fake_ccsp} with a $\chi^2$-divergence and $\lambda = 1$ as recommended by \citet{zhang2020gendice},
where $\phi^*(x) = x + \frac{x^2}{4}$.
To solve~\eqref{pb:fake_ccsp}, we need $\tau \succeq 0$.
As suggested by \citet{zhang2020gendice}, we use $(\cdot)^2$ nonlinearity.
Namely, we define $\tau(s_0, a_1) = \tau_1^2, \tau(s_0, a_2) = \tau_2^2, f(s_0, a_1) = f_1, f(s_0, a_2) = f_2$.
Now our optimization variables are $\tau_1, \tau_2, f_1, f_2, \eta$.
It is easy to verify that at the point $(\tau_1, \tau_2, f_1, f_2, \eta) = (0, 0, 0, 0, -1)$,
we have
$\frac{\partial J(\tau, f, \eta)}{\partial \tau_1} = \frac{\partial J(\tau, f, \eta)}{\partial \tau_2} = \frac{\partial J(\tau, f, \eta)}{\partial f_1} = \frac{\partial J(\tau, f, \eta)}{\partial f_2} = \frac{\partial J(\tau, f, \eta)}{\partial \eta} = 0$,
indicating GenDICE stops at this point if the true gradient is used.
However, $[0, 0]^\top$ is obviously not the optimum.
Details of the computation are provided in the appendix.
This suboptimality results from the fact that once nonlinearity is introduced in $J(\tau, f, \eta)$, it is not convex-concave even with a tabular representation.
Although this example is trivial to solve,
it numerically verifies the problems of GenDICE.

\section{GradientDICE}
As discussed above, the problems with GenDICE come mainly from the formulation of~\eqref{pb:penalty_problem},
namely the constraint $\tau \succeq 0$ and the use of the divergence $D_\phi$.
We eliminate both 
by considering the following problem instead:
\begin{align}
\label{pb:our}
\min_{\tau \in \R^{N_{sa}}} \, L(\tau) \doteq \frac{1}{2}||\mop \tau - D\tau||^2_{D^{-1}} + \frac{\lambda}{2}(d_\mu^\top \tau - 1)^2,
\end{align}

where $\lambda > 0$ is a constant and $||y||^2_\Xi$ stands for $y^\top \Xi y$.
Readers familiar with 
Gradient TD methods \citep{sutton2009fast,sutton2009convergent} or residual gradients \citep{baird1995residual} may find the first term of this objective similar to the Mean Squared Bellman Error (MSBE).
However, while in MSBE the norm is induced by $D$,
 we consider a norm induced by $D^{-1}$.
This norm is carefully designed and provides  expectations that we can sample from,
which will be clear once $L(\tau)$ is expanded (see Eq~\eqref{eq:L_tau} below).
Remarkably, we have:
\begin{theorem}
\label{thm:gradient_dice_consistency}
$\tau$ is optimal for \eqref{pb:our} iff $\tau = \tau_*$.
\end{theorem}
\begin{proof}
Sufficiency: Obviously $\tau_*$ is optimal. \\
Necessity: (a) $\gamma < 1$: 
In this case $I - \gamma P_\pi^\top$ is nonsingular. 
The linear system $D \tau - \mop \tau = 0$ has only one solution, 
we must have $\tau = \tau_*$.
(b) $\gamma=1$: If $\tau$ is optimal, we have $d_\mu^\top \tau = 1$ and $D\tau = P_\pi^\top D\tau$,
i.e., $D\tau$ is a left eigenvector of $P_\pi$ associated with the 
Perron-Frobenius eigenvalue 1.
Note $d_{\gamma}$ is also a left eigenvector of $P_\pi$ associated with the eigenvalue 1.
According to the Perron-Frobenius theorem for nonnegative irreducible matrices \citep{horn2012matrix}, 
the left eigenspace of the Perron-Frobenius eigenvalue is 1-dimensional.
Consequently, there exists a scalar $\alpha$ such that $D\tau = \alpha d_\gamma$.
On the other hand,
$\alpha = \alpha 1^\top d_\gamma = 1^\top D\tau = d_\mu^\top \tau = 1$, implying
$D\tau = d_\gamma$, i.e., $\tau = \tau_*$.
\end{proof}
\begin{remark}
Unlike the problem formulation in \citet{zhang2020gendice} (see \eqref{pb:original_problem} and \eqref{pb:penalty_problem}), 
we do not use $\tau \succeq 0$ as a constraint and can still guarantee there is no degenerate solution.
Eliminating this constraint is key to eliminating nonlinearity.
Although the Perron-Frobenius theorem can also be used in the formulation of \citet{zhang2020gendice}, 
their use of the divergence $D_\phi$ still requires $\tau \succeq 0$.
\end{remark}
With $\delta = \mop \tau - D\tau$,
we have 
\begin{align}
\label{eq:L_tau}
L(\tau) &\doteq \frac{1}{2} \E_{d_\mu}[ \big( \frac{\delta(s, a)}{d_\mu(s, a)} \big)^2 ] + \frac{\lambda}{2}(d_\mu^\top \tau - 1)^2 \\
&=\max_{f \in \R^{N_{sa}}} \E_{d_\mu}[ \frac{\delta(s, a)}{d_\mu(s, a)} f(s, a) - \frac{1}{2}f(s, a)^2] \\
& \quad + \lambda \max_{\eta \in \R} \big( \E_{d_\mu}[\eta \tau(s, a) - \eta] - \frac{\eta^2}{2} \big),
\end{align}
where the equality comes also from the Fenchel conjugate and the interchangeability principle as in \citet{zhang2020gendice}.
We, therefore, consider the following problem
\begin{align}
&\min_{\tau \in \R^{N_{sa}}} \,\, \max_{f \in \R^{N_{sa}}, \eta \in \R} \Big( L(\tau, \eta, f) \\
\doteq& \E_{d_\mu}[ \textstyle{\frac{\delta(s, a)}{d_\mu(s, a)}} f(s, a) - \frac{1}{2}f(s, a)^2] \\
&+ \lambda \big( \eta(\E_{d_\mu}[\tau(s, a) - 1]) - \textstyle{\frac{\eta^2}{2}} \big)\\
=& (1 - \gamma) \E_{\mu_0}[f(s, a)] + \gamma \E_p[\tau(s, a) f(s^\prime, a^\prime)] \\
&- \E_{d_\mu}[\tau(s, a)f(s, a)] - \frac{1}{2}\E_{d_\mu}[f(s, a)^2] \\
&+ \lambda \big( \E_{d_\mu}[\eta \tau(s, a) - \eta] - \textstyle{\frac{\eta^2}{2}} \big) \Big).
\end{align}
Here the equality comes from the fact that 
\begin{align}
\textstyle
\E_p[\tau(s, a) f(s^\prime, a^\prime)] = \sum_{s^\prime, a^\prime} (P_\pi^\top  D\tau)(s^\prime, a^\prime) f(s^\prime, a^\prime).
\end{align}
This problem is an \emph{unconstrained} optimization problem and $L$ is convex (linear) in $\tau$ and concave in $f, \eta$.
Assuming $\tau, f$ is parameterized by $w, \kappa$ respectively and
including ridge regularization for $w$ for reasons that will soon be clear,
we consider the following problem
\begin{align}
\label{pb:ccsp_linear_ridge}
\textstyle
\min_w \max_{\eta, \kappa} L(\tau_w, \eta, f_\kappa) + \frac{\xi}{2}||w||^2,
\end{align}
where $\xi \geq 0$ is a constant.
When a linear architecture is considered for $\tau_w$ and $f_\kappa$, the problem~\eqref{pb:ccsp_linear_ridge} is CCSP.
Namely, it is convex in $w$ and concave in $\kappa, \eta$.
We use $X \in \R^{N_{sa} \times K}$ to denote the feature matrix, 
each row of which is $x(s, a)$,
where $x: \mathcal{S} \times \mathcal{A} \rightarrow \R^{K}$ is the feature function.
Assuming $\tau_w \doteq Xw, f_\kappa \doteq X\kappa$,
we perform gradient descent on $w$ and gradient ascent on $\eta, \kappa$.
As we use techniques similar to Gradient TD methods to prove the convergence of our new algorithm, 
we term it Gradient stationary DIstribution Correction Estimation (GradientDICE): 
\begin{align}
\delta_t &\gets (1 - \gamma)x_{0,t} + \gamma x_t^\top w_t x^\prime_t - x_t^\top w_t x_t \\
\kappa_{t+1} &\gets \kappa_t + \alpha_t ( \delta_t  - x_t^\top \kappa_t x_t ) \\
\eta_{t+1} &\gets \eta_t + \alpha_t \lambda (x_t^\top w_t - 1 - \eta_t) \\
w_{t+1} &\gets w_t - \alpha_t \Big(\gamma x_t^{\prime\top} \kappa_t x_t - x_t^\top \kappa_t x_t + \lambda \eta_t x_t + \xi w_t \Big).
\end{align}
Here $x_{0, t} \doteq x(s_{0,t}, a_{0, t}), x_t \doteq x(s_t, a_t), x_t^\prime \doteq (s^\prime_t, a^\prime_t)$, $(s_{0,t}, a_{0, t}) \sim \mu_0, (s_t, a_t, s_t^\prime, a_t^\prime) \sim p$ (c.f. $\E_p$ in \eqref{pb:fake_ccsp}), $\{\alpha_t \}$ is a sequence of deterministic nonnegative nonincreasing learning rates satisfying the Robin-Monro's condition \citep{robbins1951stochastic}, i.e., $\sum_t \alpha_t = \infty, \sum_t \alpha_t^2 < \infty$.

\subsection{Convergence Analysis}
Let $d_t^\top \doteq [\kappa_t^\top, w_t^\top, \eta_t]$, we rewrite the GradientDICE updates as
$d_{t+1} \doteq d_t + \alpha_t (G_{t+1} d_t + g_{t+1})$,
where
\begin{align}
G_{t+1} &\doteq \begin{bmatrix}
-x_t^\top x_t & -(x_t - \gamma x_t^\prime)x_t^\top & 0 \\ 
x_t(x_t^\top - \gamma x_t^{\prime\top}) & -\xi I & -\lambda x_t \\
0 & \lambda x_t^\top & -\lambda
\end{bmatrix}, \\
g_{t+1} &\doteq \begin{bmatrix}
(1 - \gamma)x_{0, t} \\
0 \\
-\lambda \\
\end{bmatrix}.
\end{align}
Defining $A \doteq X^\top (I - \gamma P_\pi^\top) DX, C \doteq X^\top D X$, 
the limiting behavior of GradientDICE is governed by
\begin{align}
G &\doteq \E_p[G_t] = \begin{bmatrix}
-C & -A & 0 \\
A^\top & -\xi I & -\lambda X^\top d_\mu \\
0 & \lambda d_\mu^\top X & -\lambda
\end{bmatrix}, \\
\label{eq:expected_update}
g &\doteq \E_{\mu_0}[g_t] = \begin{bmatrix}
(1 - \gamma) X^\top \mu_0 \\
0 \\
-\lambda
\end{bmatrix}.
\end{align}
\begin{assumption}
\label{assu:nonsingularity}
$X$ has linearly independent columns.
\end{assumption}
\begin{assumption}
\label{assu:ridge}
A is nonsingular or $\xi > 0$.
\end{assumption}
\begin{assumption}
\label{assu:feature}
The features $\{x_{0, t}, x_t, x_t^\prime\}$ have uniformly bounded second moments.
\end{assumption}
\begin{remark}
Assumption~\ref{assu:nonsingularity} ensures $C$ is strictly positive definite.
When $\gamma < 1$, 
it is common to assume $A$ is nonsingular \citep{maei2011gradient},
the ridge regularization (i.e., $\xi > 0$) is then optional.
When $\gamma = 1$,
$A$ can easily be singular (e.g., in a tabular setting).
We, therefore, impose the extra ridge regularization.
Assumption~\ref{assu:feature} is commonly used in Gradient TD methods \citep{maei2011gradient}.
Assumptions~(\ref{assu:nonsingularity} - \ref{assu:feature}) are also used in previous density ratio learning literature for analyzing the optimization error, 
either explicitly (e.g., the Assumption 8 in Appendix D.4 in \citet{nachum2019dualdice})
or implicitly (e.g., Appendix C.3 in \citet{zhang2020gendice}).
\end{remark}
\begin{theorem}
\label{thm:gen_dice_plus_convergence}
Under Assumptions~(\ref{assu:nonsingularity}-\ref{assu:feature}), 
we have
$$\lim_{t \rightarrow \infty} d_t = -G^{-1}g \quad \text{almost surely}. $$
\end{theorem}
We provide a detailed proof of Theorem~\ref{thm:gen_dice_plus_convergence} in the appendix, 
which is inspired by \citet{sutton2009fast}.
One key step in the proof is to show that the real parts of all eigenvalues of $G$ are strictly negative.
The $G$ in \citet{sutton2009fast} satisfies this condition easily.
However, for our $G$ to satisfy this condition when $\gamma = 1$, 
we must have $\xi > 0$,
which motivates the use of ridge regularization.

With simple block matrix inversion expanding $G^{-1}$, we have
$\lim_{t \rightarrow \infty} w_t = w_{\infty, \xi}$,
where
\begin{align}
\label{eq:w_infinity}
w_{\infty, \xi} &\doteq (1 - \gamma)\Xi A^\top C^{-1}X^\top \mu_0 \\
& \quad+ \lambda \beta^{-1}z[1 - (1 - \gamma) z^\top A^\top C^{-1} X^\top \mu_0]\\
\Xi &\doteq (\xi I + A^\top C^{-1}A)^{-1}, \\
z &\doteq  \Xi X^\top d_\mu, \quad \beta \doteq 1 + \lambda d_\mu^\top X \Xi X^\top d_\mu.
\end{align}
The maximization step in~\eqref{pb:ccsp_linear_ridge} is quadratic (with linear function approximation) and
thus can be solved analytically.
Simple algebraic manipulation together with Assumption~\ref{assu:nonsingularity} shows that this quadratic problem has a unique optimizer for all $\gamma \in [0, 1]$.
Plugging the analytical solution for the maximization step in~\eqref{pb:ccsp_linear_ridge}, the
KKT conditions then state that the optimizer $w_{*, \xi} $ for the minimization step must satisfy $A_{*, \xi} w_{*, \xi} = b_*$, where
\begin{align}
A_{*, \xi} &\doteq A^\top C^{-1}A + \lambda X^\top d_\mu d_\mu^\top X + \xi I, \\
b_* &\doteq (1 - \gamma) A^\top C^{-1} X^\top \mu_0 + \lambda X^\top d_\mu.
\end{align}
Assumption~\ref{assu:ridge} ensures $A_{*, \xi}$ is nonsingular.
Using the Sherman-Morrison formula \citep{sherman1950adjustment},
it is easy to verify $w_{*, \xi} = w_{\infty, \xi}$. 
For a quick sanity check, 
it is easy to verify that $w_{\infty, 0} = \tau_*$ holds when $\gamma < 1$, $X = I$, and $\xi = 0$,
using the fact $(1-\gamma)1^\top (I - \gamma P_\pi^\top) \mu_0 = 1$.

\subsection{Consistency Analysis}
To ensure convergence, 
we require ridge regularization in~\eqref{pb:ccsp_linear_ridge} for the setting $\gamma = 1$.
The asymptotic solution $w_{\infty, \xi}$ is therefore biased. 
We now study the regularization path consistency for the setting $\gamma = 1$, i.e., 
we study the behavior of $w_{\infty, \xi}$ when $\xi$ approaches 0.


\textbf{Case 1:} $\tau_* \in col(X)$. 
Here $col(\cdot)$ indicates the column space.
As $X$ has linearly independent columns (Assumption~\ref{assu:nonsingularity}),
we use $w_*$ to denote the unique $w$ satisfying $X w = \tau_*$. 
As $\gamma = 1$, $A$ can be singular.
Hence both $w_{\infty, 0}$ and $A_{*, 0}^{-1}$ can be ill-defined.
We now show under some regularization, 
we still have the desired consistency.
As $A^\top C^{-1}A$ is always positive semidefinite,
we consider its eigendecomposition $A^\top C^{-1}A = Q^\top \Lambda Q$, 
where $Q$ is an orthogonal matrix, $\Lambda \doteq diag([\lambda_1, \cdots, \lambda_r, 0, \cdots, 0])$, $r$ is the rank of $A^\top C^{-1}A$, $\lambda_i > 0$ are eigenvalues.
Let $u \doteq QX^\top d_\mu$, we have 
\begin{proposition}
Assuming $XC^{-1}X^\top$ is positive definite, $||u_{r+1:N_{sa}}|| \neq 0$, then $\lim_{\xi \rightarrow 0} w_{\infty, \xi} = w_{*},$ where $u_{i:j}$ denotes the vector consisting of the elements indexed by $i, i+1, \dots, j$ in the vector $u$.
\end{proposition}
\begin{proof}
According to the Perron-Frobenius theorem (c.f. the proof of Theorem~\ref{thm:gradient_dice_consistency}), 
it suffices to show
\begin{align}
\textstyle
&\lim_{\xi \rightarrow 0} L_1(w_{\infty, \xi}) = \lim_{\xi \rightarrow 0} L_2(w_{\infty, \xi}) = 0, \\
&L_1(w_{\infty, \xi}) \doteq d_\mu^\top X w_{\infty, \xi} - 1, \\
&L_2(w_{\infty, \xi}) \doteq ||DXw_{\infty, \xi} - P_\pi^\top XDw_{\infty, \xi}||_{XC^{-1}X^\top}^2,
\end{align}
as $w_*$ is the only $w$ satisfying $L_1(w) = L_2(w) = 0$.
With the eigendecomposition of $A^\top C^{-1}A$, we can compute $\Xi$ explicitly.
Simple algebraic manipulation then yields
\begin{align}
L_1(w_{\infty, \xi}) &= \textstyle{\frac{\lambda u^\top \Lambda_\xi u}{1 + \lambda u^\top \Lambda_\xi u}} - 1, \\
L_2(w_{\infty, \xi}) &= \textstyle{\frac{\lambda^2 u^\top \Lambda_\xi u}{(1 + \lambda u^\top \Lambda_\xi u)^2}} + \textstyle{\frac{\lambda^2 \xi u^\top \Lambda_\xi^2 u}{(1 + \lambda u^\top \Lambda_\xi u)^2}},
\end{align}
where $\Lambda_\xi \doteq diag([\frac{1}{\xi + \lambda_1}, \dots, \frac{1}{\xi + \lambda_r}, \frac{1}{\xi}, \cdots, \frac{1}{\xi}])$.
The desired limits then follow from the L'Hopital's rule.
\end{proof}
\begin{remark}
The assumption $||u_{r+1:N_{sa}}|| \neq 0$ is not restrictive as it is independent of learnable parameters and mainly controlled by features.
Requiring $XC^{-1}X^\top$ to be positive definite is more restrictive, but it holds at least for the tabular setting (i.e., $X = I$). 
The difficulty of the setting $\gamma = 1$ comes mainly from the fact that the objective of the minimization step in the problem \eqref{pb:ccsp_linear_ridge} is no longer strictly convex when $\xi = 0$ (i.e., $A_{*, 0}$ can be singular).
Thus there may be multiple optima for this minimization step, only one of which is $w_{*}$.
Extra domain knowledge (e.g., assumptions in the proposition statement) is necessary to ensure the regularization path converges to the desired optimum.
We provide a sufficient condition here and leave the analysis of necessary conditions for future work.
\end{remark}

\textbf{Case 2:} $\tau_* \notin col(X).$ 
In this scenario, it is not clear how to define $w_{*}$. 
The minimization step in \eqref{pb:ccsp_linear_ridge} can have multiple optima, and it is not clear which one is the best.
To analyze this scenario, we need to explicitly define projection in the optimization objective like Mean Squared Projected Bellman Error \citep{sutton2009fast}, 
instead of using an MSBE-like objective.
We leave this for future work.

\begin{figure*}[h]
\centering
\includegraphics[width=0.8\textwidth]{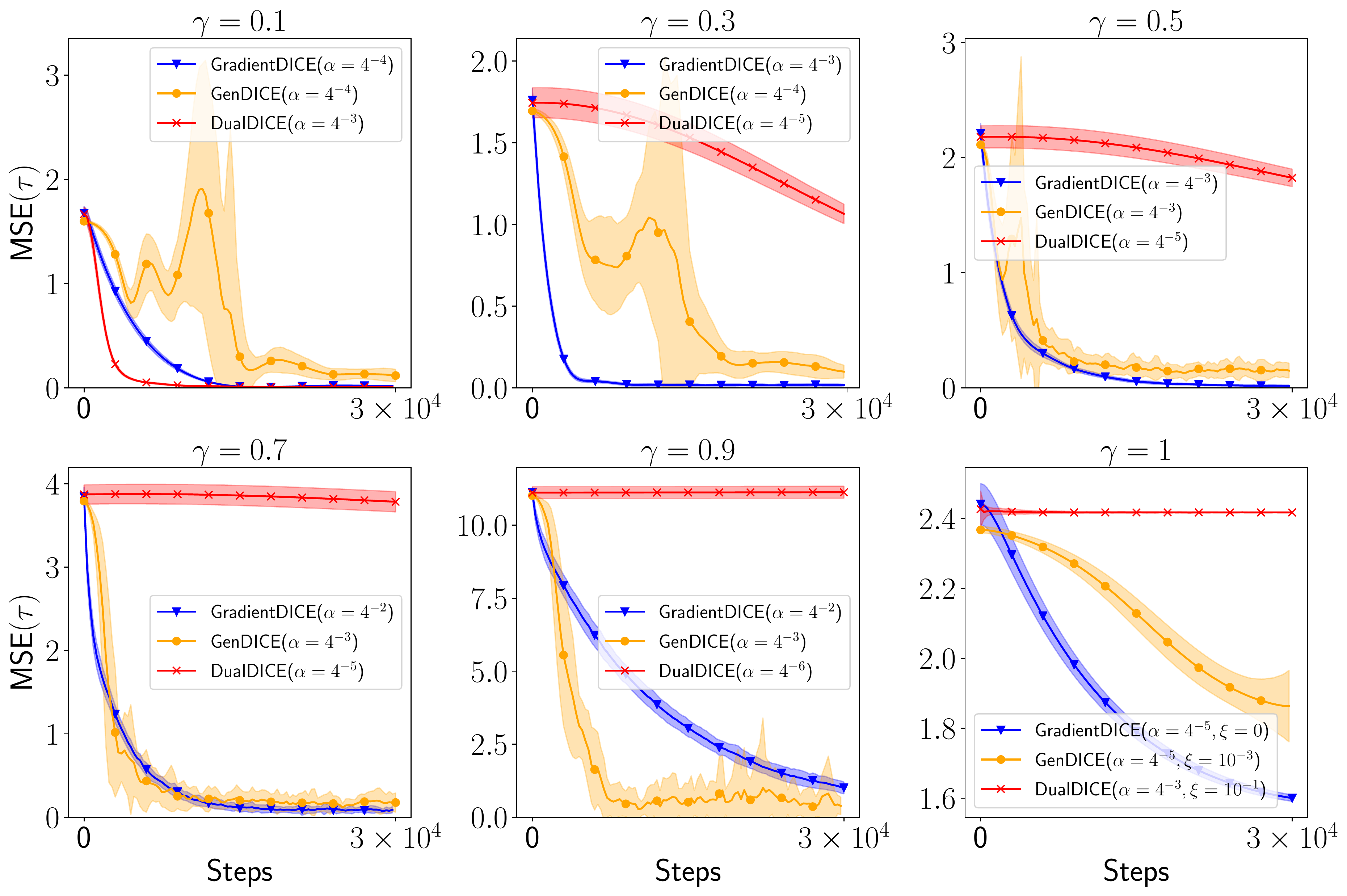}
\caption{\label{fig:boyans_chain_tabular} Density ratio learning in Boyan's Chain with a tabular representation.}
\end{figure*}

\begin{figure*}[h]
\centering
\includegraphics[width=0.8\textwidth]{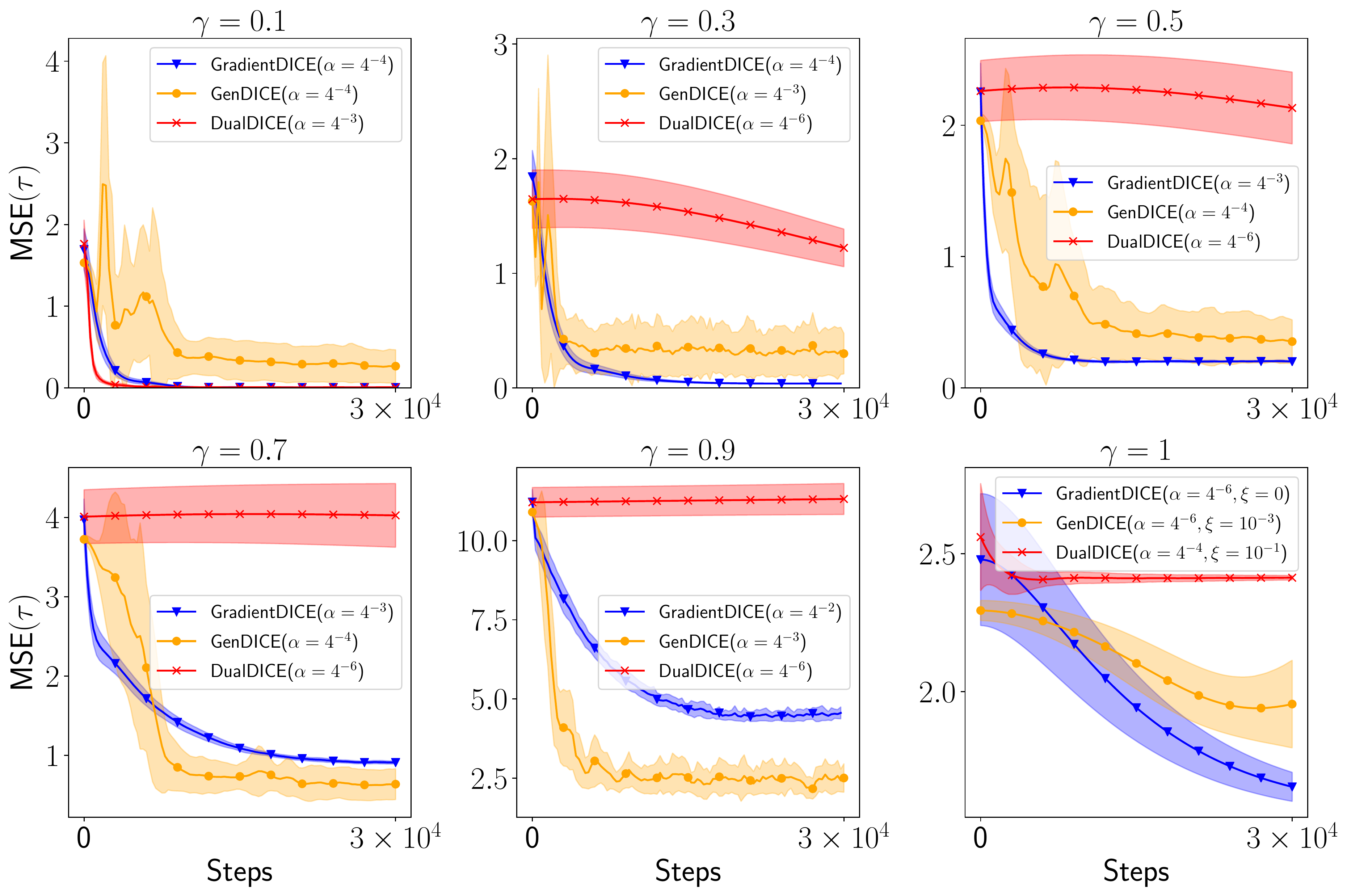}
\caption{\label{fig:boyans_chain_linear} Density ratio learning in Boyan's Chain with a linear architecture.}
\end{figure*}

\subsection{Finite Sample Analysis}
We now provide a finite sample analysis for a variant of GradientDICE, \emph{Projected GradientDICE} (Algorithm~\ref{alg:revised-gdice}), 
where we introduce projection and iterates average, 
akin to \citet{nemirovski2009robust,liu2015finite}. 
\begin{algorithm}
\caption{Projected GradientDICE}
\label{alg:revised-gdice}
\begin{algorithmic}
\FOR{$t = 0, \dots, n-1$}
\STATE $y_{t+1} \gets \Pi_Y [y_t + \alpha_t (G_{1, t} y_t + G_{2, t} w_t + G_{5, t})]$
\STATE $w_{t+1} \gets \Pi_W [w_t + \alpha_t (G_{3, t} y_t + G_{4, t} w_t)]$
\ENDFOR
\STATE Output:
\STATE $\bar{y}_n \doteq \frac{\sum_{t=1}^n \alpha_t y_t}{\sum_{t=1}^n \alpha_t}, \bar{w}_n \doteq \frac{\sum_{t=1}^n \alpha_t w_t}{\sum_{t=1}^n \alpha_t}$
\end{algorithmic}
\end{algorithm}
Intuitively, Projected GradientDICE groups the $\kappa, \eta$ in GradientDICE into $y^\top = [\kappa^\top, \eta]$.
Precisely, we have $y_t \in \R^{K+1}, w_t \in \R^K$,
\begin{align}
G_{1,t} &\doteq \begin{bmatrix}
-x_t^\top x_t & 0 \\
0 & -\lambda 
\end{bmatrix}, G_{2, t} \doteq \begin{bmatrix}
-(x_t - \gamma x_t^\prime)x_t^\top \\ \lambda x_t^\top
\end{bmatrix}, \\
G_{3, t} &\doteq \begin{bmatrix}
x_t(x_t^\top - \gamma x_t^{\prime\top}) & -\lambda x_t
\end{bmatrix}, G_{4, t} \doteq -\xi I, \\
G_{5,t} &\doteq \begin{bmatrix}
(1 - \gamma) x_{0, t} \\ -\lambda
\end{bmatrix},
\end{align}
$Y \subset \R^{K+1}$, $W \subset \R^{K}$. $\Pi_Y$ and $\Pi_W$ are projections onto $Y$ and $W$ w.r.t. $\ell_2$ norm (such projection is trivial if $Y$ and $W$ are balls), $n$ is the number of iterations, 
and $\alpha_t$ is a learning rate, detailed below.
We consider the following problem
\begin{align}
\label{pb:finite_sample}
\min_{w \in W} \max_{y \in Y} \big( L(w, y) \doteq L(\tau_w, \eta, f_\kappa) + \frac{\xi}{2}||w||^2 \big).
\end{align}
It is easy to see $L(w, y)$ is a convex-concave function and its saddle point $(w^*, y^*)$ is unique.
We assume
\begin{assumption}
\label{assu:feasibility}
$Y$ and $W$ are bounded, closed, and convex, $w^* \in W, y^* \in Y$.
\end{assumption} 
For the CCSP problem~\eqref{pb:finite_sample},
we define the optimization error
\begin{align}
\textstyle
\eopt(w, y) \doteq \max_{y^\prime \in Y} L(w, y^\prime) - \min_{w^\prime \in W} L(w^\prime, y).
\end{align}
It is easy to see $L(w, y) = 0$ iif $(w, y) = (w^*, y^*)$.
\begin{proposition}
\label{prop:opt_error}
Under Assumptions~(\ref{assu:nonsingularity}-\ref{assu:feasibility}),
for the $(\bar{w}_n, \bar{y}_n)$ from the Projected GradientDICE algorithm after $n$ iterations,
we have at least with probability $1 - \delta$
\begin{align}
\textstyle
\eopt(\bar{w}_n, \bar{y}_n) \leq \sqrt{\frac{5}{n}}(8 + 2\ln \frac{2}{\delta} )C_0,
\end{align}
where $C_0 > 0$ is a constant.
\end{proposition}
Both $C_0$ and the learning rates $\alpha_t$ are detailed in the proof in the appendix.
Note it is possible to conduct a finite sample analysis without introducing projection using arguments from \citet{lakshminarayanan2018linear},
which we leave for future work.

\section{Experiments}

In this section, we present experiments comparing GradientDICE to GenDICE and DualDICE.
All curves are averaged over 30 independent runs and shaded regions indicate one standard derivation.
The implementations are made publicly available for future research.\footnote{\url{https://github.com/ShangtongZhang/DeepRL}}

\subsection{Density Ratio Learning}

\begin{figure}[h]
\centering
\includegraphics[width=0.5\textwidth]{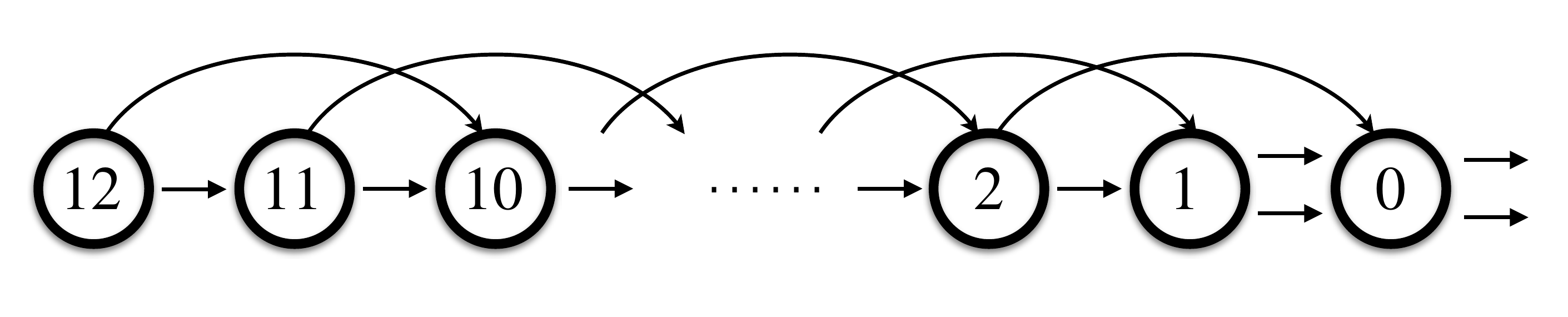}
\caption{\label{fig:boyans_chain} Two variants of Boyan's Chain. 
There are 13 states in total with two actions $\{a_0, a_1\}$ available at each state. 
The initial distribution $\tilde{\mu}_0$ is uniform over $\{s_0, \dots, s_{12}\}$.
At a state $s_i (i\geq 2)$, $a_0$ leads to $s_{i-1}$ and $a_1$ leads to $s_{i-2}$. 
At $s_1$, both actions leads to $s_0$.
At $s_0$, there are two variants. 
(1) \texttt{Episodic Boyan's Chain}: both actions at $s_0$ lead to $s_0$ itself, i.e., $s_0$ is an absorbing state.
(2) \texttt{Continuing Boyan's Chain}: both actions at $s_0$ lead to a random state among $\{s_0, \dots, s_{12}\}$ with equal probability.}
\end{figure}

\begin{figure*}[h]
\centering
\includegraphics[width=0.8\textwidth]{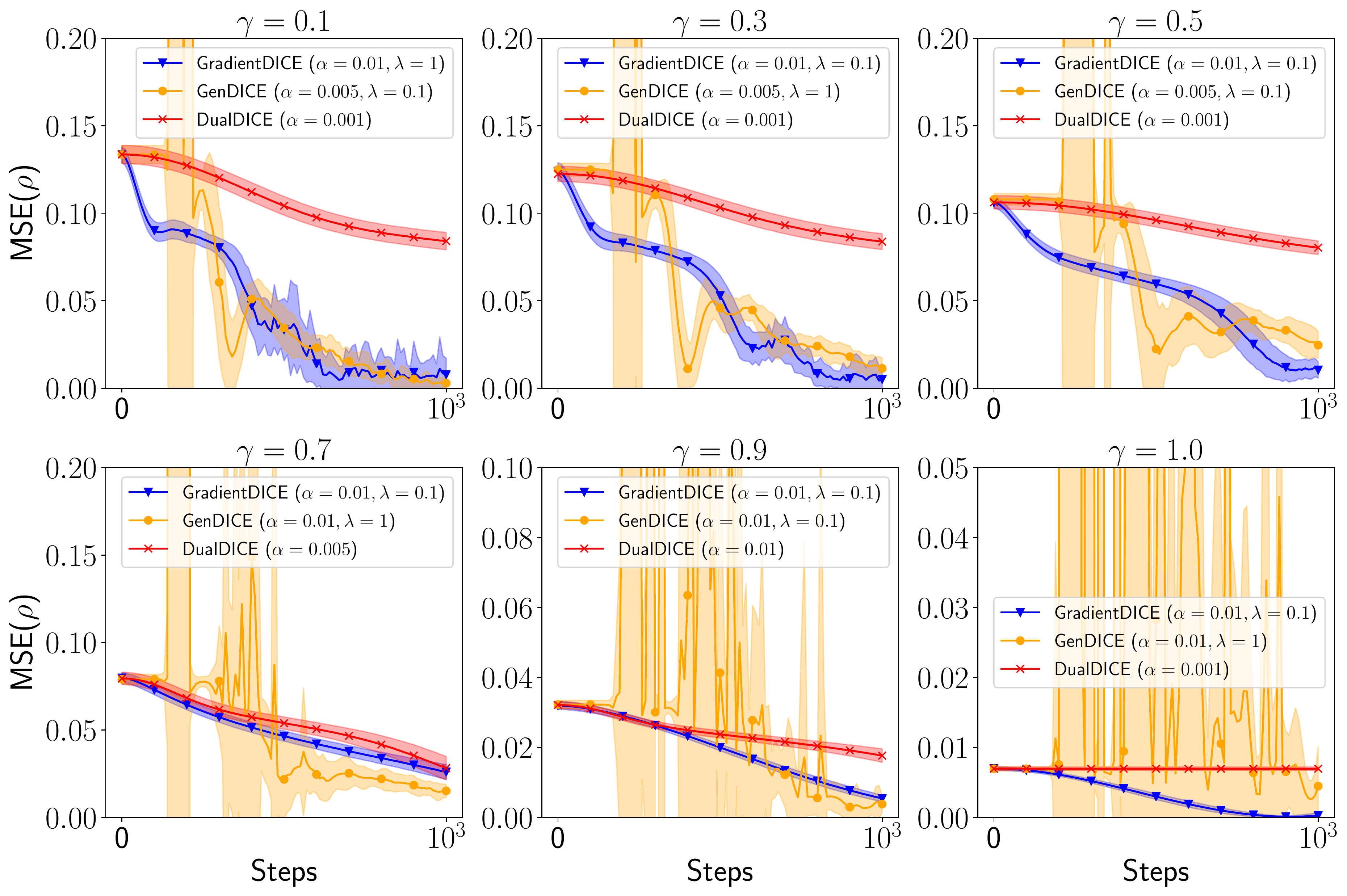}
\caption{\label{fig:mujoco_1} Off-policy evaluation in \texttt{Reacher-v2} with neural network function approximators.} 
\end{figure*}

We consider two variants of Boyan's Chain \citep{boyan1999least} as shown in Figure~\ref{fig:boyans_chain}.
In particular, we use \texttt{Episodic Boyan's Chain} when $\gamma < 1$ and \texttt{Continuing Boyan's Chain} when $\gamma = 1$. 
We consider a uniform sampling distribution, i.e., $d_\mu(s, a) = \frac{1}{26} \, \forall (s, a)$, and
a target policy $\pi$ satisfying $\pi(a_0|s) = 0.1 \, \forall s$.
We design a sequence of tasks by varying the discount factor $\gamma$ in $\{0.1, 0.3, 0.5, 0.7, 0.9, 1\}$.

We train all compared algorithms for $3 \times 10^{4}$ steps.
We evaluate the Mean Squared Error (MSE) for the predicted $\tau$ every 300 steps, computed as $\text{MSE}(\tau) \doteq \frac{1}{26}\sum_{s, a}(\tau(s, a) - \tau_*(s, a))^2$,
where the ground truth $\tau_*$ is computed analytically.
We use fixed learning rates $\alpha$ for all algorithms,
which is tuned from $\{4^{-6}, 4^{-5}, \dots, 4^{-1}\}$ to minimize the MSE($\tau$) at the end of training.
For the setting $\gamma = 1$, 
we additionally tune $\xi$ from $\{0, 10^{-3}, 10^{-2}, 10^{-1}\}$ 
(for a fair comparison, we also add this ridge regularization for GenDICE and DualDICE). 
For the penalty coefficient, we set $\lambda = 1$ as recommended by \citet{zhang2020gendice}.
We find $\lambda$ has little influence on the learning process in this domain.

We report the results in both tabular (Figure~\ref{fig:boyans_chain_tabular}) and linear (Figure~\ref{fig:boyans_chain_linear}) settings.
In the tabular setting, we use lookup tables to store $\tau, f$ and $\eta$.
In the linear setting, we use two independent sets of weights for the two actions.
As GenDICE requires $\tau \succeq 0$,
we use the nonlinearity $(\cdot)^2$ for its $\tau$ prediction as suggested by \citet{zhang2020gendice}.
We do not apply any nonlinearity for GradientDICE and DualDICE.
Our results show that GradientDICE reaches a lower prediction error at the end of training than GenDICE in 4 (5) out 6 tasks in the tabular (linear) setting.
Moreover, the learning curves of GradientDICE are more stable than those of GenDICE in all the 6 tasks in both tabular and linear settings.
Although DualDICE performs the best for the task $\gamma = 0.1$, 
it becomes unstable as $\gamma$ increases,
which is also observed in \citet{zhang2020gendice}.

\subsection{Off-Policy Evaluation}
We now benchmark DualDICE, GenDICE, and GradientDICE in an off-policy evaluation problem.  
We consider \texttt{Reacher-v2} from OpenAI Gym \citep{brockman2016openai}.
We consider policies in the form of $\pi_d(s, a) + \mathcal{N}(0, \sigma^2)$,
where $\pi_d$ is a deterministic policy trained via TD3 \citep{fujimoto2018addressing} for $10^6$ steps and $\mathcal{N}$ is Gaussian noise.
For the behavior policy, 
we set $\sigma = 0.1$ and run the policy for $10^5$ steps to collect transitions,
which form the dataset used across all the experiments.
For the target policy,
we set $\sigma = 0.05$.

We use neural networks to parameterize $\tau$ and $f$, 
each of which is represented by a two-hidden-layer network with 64 hidden units and ReLU \citep{nair2010rectified} activation function.
For GenDICE, we add $(\cdot)^2$ nonlinearity for the $\tau$ prediction by the network.
For GradientDICE and DualDICE, 
we do not have such nonlinearity in their $\tau$ prediction.
Given the learned $\tau$, 
the performance $\rho_\gamma(\pi)$ is approximated by $\hat{\rho}_\gamma(\pi) \doteq \frac{1}{N} \sum_{i=1}^N \tau(s_i, a_i)r_i$.
We train each algorithm for $10^3$ steps and examine MSE$(\rho) \doteq \frac{1}{2}(\rho_\gamma(\pi) - \hat{\rho}_\gamma(\pi))^2$ every 10 steps, 
where the ground truth $\rho_\gamma(\pi)$ is computed from Monte Carlo methods via executing the target policy $\pi$ multiple times.
We use SGD to train the neural networks with batch size 128.
The learning rate $\alpha$ and the penalty coefficient $\lambda$ are tuned from $\{0.01, 0.005, 0.001\}$ and $\{0.1, 1\}$ with grid search to minimize MVE($\rho$) at the end of training.

The results are reported in Figure~\ref{fig:mujoco_1}.
Although policy evaluation errors of GradientDICE and GenDICE tend to be similar at the end of training,
the learning curves of GradientDICE are more stable than those of GenDICE, 
which matches the results in the tabular and linear settings.
Although DualDICE tends to be more stable than both GradientDICE and GenDICE, 
it learns slower and does not work for the setting $\gamma = 1$,
which also matches the results in \citet{zhang2020gendice}.  
To summarize, GradientDICE combines the advantages of both DualDICE (stability under the total discounted reward criterion) and GenDICE (compatibility with the average reward criterion).

\section{Related Work}
Inspired by \citet{hallak2017consistent,liu2018breaking}, 
\citet{gelada2019off,uehara2019minimax} propose different estimators for learning density ratios, 
either with semi-gradient updates \citep{sutton1988learning} or by restricting the function class to Reproducing Kernel Hilbert Space.
\citet{zhang2019provably} show that some density ratios can be interpreted as special emphasis \citep{sutton2016emphatic}.
Hence the Gradient Emphasis Learning algorithm in \citet{zhang2019provably} can also be used to learn certain density ratios.
All these methods, however, require a single known behavior policy.  
By contrast, DualDICE, GenDICE, and GradientDICE cope well with multiple unknown behavior policies.


\section{Conclusion}
In this paper, we point out two problems with GenDICE and fix them with GradientDICE. 
We provide a comprehensive theoretical analysis for GradientDICE.
Our experiments confirm that the theoretical advantages of GradientDICE over GenDICE translate into an empirical performance boost in the tested domains.
Overall, our work provides a new theoretical justification for the field of density-ratio-learning-based off-policy evaluation.

\section*{Acknowledgments}
SZ is generously funded by the Engineering and Physical Sciences Research Council (EPSRC). This project has received funding from the European Research Council under the European Union's Horizon 2020 research and innovation programme (grant agreement number 637713). The experiments were made possible by a generous equipment grant from NVIDIA. 
BL’s research is funded by the National Science Foundation (NSF) under grant NSF IIS1910794, Amazon Research Award, and Adobe gift fund.
The authors thank Lihong Li and Bo Dai for the useful discussion.

\bibliography{ref}

\begin{thebibliography}{38}
\providecommand{\natexlab}[1]{#1}
\providecommand{\url}[1]{\texttt{#1}}
\expandafter\ifx\csname urlstyle\endcsname\relax
  \providecommand{\doi}[1]{doi: #1}\else
  \providecommand{\doi}{doi: \begingroup \urlstyle{rm}\Url}\fi

\bibitem[Baird(1995)]{baird1995residual}
Baird, L.
\newblock Residual algorithms: Reinforcement learning with function
  approximation.
\newblock \emph{Machine Learning}, 1995.

\bibitem[Beck \& Teboulle(2003)Beck and Teboulle]{mid:beck2003}
Beck, A. and Teboulle, M.
\newblock Mirror descent and nonlinear projected subgradient methods for convex
  optimization.
\newblock \emph{Operations Research Letters}, 31\penalty0 (3):\penalty0
  167--175, 2003.

\bibitem[Boyan(1999)]{boyan1999least}
Boyan, J.~A.
\newblock Least-squares temporal difference learning.
\newblock In \emph{Proceedings of the 16th International Conference on Machine
  Learning}, 1999.

\bibitem[Boyd \& Vandenberghe(2004)Boyd and Vandenberghe]{boyd2004convex}
Boyd, S. and Vandenberghe, L.
\newblock \emph{Convex optimization}.
\newblock Cambridge university press, 2004.

\bibitem[Brockman et~al.(2016)Brockman, Cheung, Pettersson, Schneider,
  Schulman, Tang, and Zaremba]{brockman2016openai}
Brockman, G., Cheung, V., Pettersson, L., Schneider, J., Schulman, J., Tang,
  J., and Zaremba, W.
\newblock Openai gym.
\newblock \emph{arXiv preprint arXiv:1606.01540}, 2016.

\bibitem[Dai et~al.(2016)Dai, He, Pan, Boots, and Song]{dai2016learning}
Dai, B., He, N., Pan, Y., Boots, B., and Song, L.
\newblock Learning from conditional distributions via dual embeddings.
\newblock \emph{arXiv preprint arXiv:1607.04579}, 2016.

\bibitem[Fujimoto et~al.(2018)Fujimoto, van Hoof, and
  Meger]{fujimoto2018addressing}
Fujimoto, S., van Hoof, H., and Meger, D.
\newblock Addressing function approximation error in actor-critic methods.
\newblock \emph{arXiv preprint arXiv:1802.09477}, 2018.

\bibitem[Gelada \& Bellemare(2019)Gelada and Bellemare]{gelada2019off}
Gelada, C. and Bellemare, M.~G.
\newblock Off-policy deep reinforcement learning by bootstrapping the covariate
  shift.
\newblock In \emph{Proceedings of the 33rd AAAI Conference on Artificial
  Intelligence}, 2019.

\bibitem[Hallak \& Mannor(2017)Hallak and Mannor]{hallak2017consistent}
Hallak, A. and Mannor, S.
\newblock Consistent on-line off-policy evaluation.
\newblock In \emph{Proceedings of the 34th International Conference on Machine
  Learning}, 2017.

\bibitem[Horn \& Johnson(2012)Horn and Johnson]{horn2012matrix}
Horn, R.~A. and Johnson, C.~R.
\newblock \emph{Matrix analysis (2nd Edition)}.
\newblock Cambridge university press, 2012.

\bibitem[Jiang \& Li(2015)Jiang and Li]{jiang2015doubly}
Jiang, N. and Li, L.
\newblock Doubly robust off-policy value evaluation for reinforcement learning.
\newblock \emph{arXiv preprint arXiv:1511.03722}, 2015.

\bibitem[Lakshminarayanan \& Szepesvari(2018)Lakshminarayanan and
  Szepesvari]{lakshminarayanan2018linear}
Lakshminarayanan, C. and Szepesvari, C.
\newblock Linear stochastic approximation: How far does constant step-size and
  iterate averaging go?
\newblock In \emph{The 21st International Conference on Artificial Intelligence
  and Statistics}, 2018.

\bibitem[Lin(1992)]{lin1992self}
Lin, L.-J.
\newblock Self-improving reactive agents based on reinforcement learning,
  planning and teaching.
\newblock \emph{Machine Learning}, 1992.

\bibitem[Liu et~al.(2015)Liu, Liu, Ghavamzadeh, Mahadevan, and
  Petrik]{liu2015finite}
Liu, B., Liu, J., Ghavamzadeh, M., Mahadevan, S., and Petrik, M.
\newblock Finite-sample analysis of proximal gradient td algorithms.
\newblock In \emph{UAI}, 2015.

\bibitem[Liu et~al.(2018)Liu, Li, Tang, and Zhou]{liu2018breaking}
Liu, Q., Li, L., Tang, Z., and Zhou, D.
\newblock Breaking the curse of horizon: Infinite-horizon off-policy
  estimation.
\newblock In \emph{Advances in Neural Information Processing Systems}, 2018.

\bibitem[Maei(2011)]{maei2011gradient}
Maei, H.~R.
\newblock \emph{Gradient temporal-difference learning algorithms}.
\newblock PhD thesis, University of Alberta, 2011.

\bibitem[Mousavi et~al.(2020)Mousavi, Li, Liu, and Zhou]{mousavi2020blackbox}
Mousavi, A., Li, L., Liu, Q., and Zhou, D.
\newblock Black-box off-policy estimation for infinite-horizon reinforcement
  learning.
\newblock In \emph{International Conference on Learning Representations}, 2020.

\bibitem[Nachum et~al.(2019)Nachum, Chow, Dai, and Li]{nachum2019dualdice}
Nachum, O., Chow, Y., Dai, B., and Li, L.
\newblock Dualdice: Behavior-agnostic estimation of discounted stationary
  distribution corrections.
\newblock \emph{arXiv preprint arXiv:1906.04733}, 2019.

\bibitem[Nair \& Hinton(2010)Nair and Hinton]{nair2010rectified}
Nair, V. and Hinton, G.~E.
\newblock Rectified linear units improve restricted boltzmann machines.
\newblock In \emph{Proceedings of the 27th International Conference on Machine
  Learning}, 2010.

\bibitem[Nemirovski et~al.(2009)Nemirovski, Juditsky, Lan, and
  Shapiro]{nemirovski2009robust}
Nemirovski, A., Juditsky, A., Lan, G., and Shapiro, A.
\newblock Robust stochastic approximation approach to stochastic programming.
\newblock \emph{SIAM Journal on optimization}, 2009.

\bibitem[Nowozin et~al.(2016)Nowozin, Cseke, and Tomioka]{nowozin2016f}
Nowozin, S., Cseke, B., and Tomioka, R.
\newblock f-gan: Training generative neural samplers using variational
  divergence minimization.
\newblock In \emph{Advances in neural information processing systems}, pp.\
  271--279, 2016.

\bibitem[Precup et~al.(2001)Precup, Sutton, and Dasgupta]{precup2001off}
Precup, D., Sutton, R.~S., and Dasgupta, S.
\newblock Off-policy temporal-difference learning with function approximation.
\newblock In \emph{Proceedings of the 18th International Conference on Machine
  Learning}, 2001.

\bibitem[Puterman(2014)]{puterman2014markov}
Puterman, M.~L.
\newblock \emph{Markov decision processes: discrete stochastic dynamic
  programming}.
\newblock John Wiley \& Sons, 2014.

\bibitem[Robbins \& Monro(1951)Robbins and Monro]{robbins1951stochastic}
Robbins, H. and Monro, S.
\newblock A stochastic approximation method.
\newblock \emph{The Annals of Mathematical Statistics}, 1951.

\bibitem[Shapiro et~al.(2014)Shapiro, Dentcheva, and
  Ruszczy{\'n}ski]{shapiro2014lectures}
Shapiro, A., Dentcheva, D., and Ruszczy{\'n}ski, A.
\newblock \emph{Lectures on stochastic programming: modeling and theory}.
\newblock SIAM, 2014.

\bibitem[Sherman \& Morrison(1950)Sherman and Morrison]{sherman1950adjustment}
Sherman, J. and Morrison, W.~J.
\newblock Adjustment of an inverse matrix corresponding to a change in one
  element of a given matrix.
\newblock \emph{The Annals of Mathematical Statistics}, 1950.

\bibitem[Sutton(1988)]{sutton1988learning}
Sutton, R.~S.
\newblock Learning to predict by the methods of temporal differences.
\newblock \emph{Machine Learning}, 1988.

\bibitem[Sutton \& Barto(2018)Sutton and Barto]{sutton2018reinforcement}
Sutton, R.~S. and Barto, A.~G.
\newblock \emph{Reinforcement Learning: An Introduction (2nd Edition)}.
\newblock MIT press, 2018.

\bibitem[Sutton et~al.(2009{\natexlab{a}})Sutton, Maei, Precup, Bhatnagar,
  Silver, Szepesv{\'a}ri, and Wiewiora]{sutton2009fast}
Sutton, R.~S., Maei, H.~R., Precup, D., Bhatnagar, S., Silver, D.,
  Szepesv{\'a}ri, C., and Wiewiora, E.
\newblock Fast gradient-descent methods for temporal-difference learning with
  linear function approximation.
\newblock In \emph{Proceedings of the 26th International Conference on Machine
  Learning}, 2009{\natexlab{a}}.

\bibitem[Sutton et~al.(2009{\natexlab{b}})Sutton, Maei, and
  Szepesv{\'a}ri]{sutton2009convergent}
Sutton, R.~S., Maei, H.~R., and Szepesv{\'a}ri, C.
\newblock A convergent $ o (n) $ temporal-difference algorithm for off-policy
  learning with linear function approximation.
\newblock In \emph{Advances in Neural Information Processing Systems},
  2009{\natexlab{b}}.

\bibitem[Sutton et~al.(2011)Sutton, Modayil, Delp, Degris, Pilarski, White, and
  Precup]{sutton2011horde}
Sutton, R.~S., Modayil, J., Delp, M., Degris, T., Pilarski, P.~M., White, A.,
  and Precup, D.
\newblock Horde: A scalable real-time architecture for learning knowledge from
  unsupervised sensorimotor interaction.
\newblock In \emph{Proceedings of the 10th International Conference on
  Autonomous Agents and Multiagent Systems}, 2011.

\bibitem[Sutton et~al.(2016)Sutton, Mahmood, and White]{sutton2016emphatic}
Sutton, R.~S., Mahmood, A.~R., and White, M.
\newblock An emphatic approach to the problem of off-policy temporal-difference
  learning.
\newblock \emph{The Journal of Machine Learning Research}, 2016.

\bibitem[Tsitsiklis \& Van~Roy(1997)Tsitsiklis and
  Van~Roy]{tsitsiklis1997analysis}
Tsitsiklis, J.~N. and Van~Roy, B.
\newblock Analysis of temporal-diffference learning with function
  approximation.
\newblock In \emph{Advances in Neural Information Processing Systems}, 1997.

\bibitem[Uehara \& Jiang(2019)Uehara and Jiang]{uehara2019minimax}
Uehara, M. and Jiang, N.
\newblock Minimax weight and q-function learning for off-policy evaluation.
\newblock \emph{arXiv preprint arXiv:1910.12809}, 2019.

\bibitem[Wang et~al.(2007)Wang, Bowling, and Schuurmans]{wang2007dual}
Wang, T., Bowling, M., and Schuurmans, D.
\newblock Dual representations for dynamic programming and reinforcement
  learning.
\newblock In \emph{2007 IEEE International Symposium on Approximate Dynamic
  Programming and Reinforcement Learning}, 2007.

\bibitem[Wang et~al.(2008)Wang, Bowling, Schuurmans, and
  Lizotte]{wang2008stable}
Wang, T., Bowling, M., Schuurmans, D., and Lizotte, D.~J.
\newblock Stable dual dynamic programming.
\newblock In \emph{Advances in neural information processing systems}, 2008.

\bibitem[Zhang et~al.(2020{\natexlab{a}})Zhang, Dai, Li, and
  Schuurmans]{zhang2020gendice}
Zhang, R., Dai, B., Li, L., and Schuurmans, D.
\newblock Gendice: Generalized offline estimation of stationary values.
\newblock In \emph{International Conference on Learning Representations},
  2020{\natexlab{a}}.

\bibitem[Zhang et~al.(2020{\natexlab{b}})Zhang, Liu, Yao, and
  Whiteson]{zhang2019provably}
Zhang, S., Liu, B., Yao, H., and Whiteson, S.
\newblock Provably convergent two-timescale off-policy actor-critic with
  function approximation.
\newblock In \emph{Proceedings of the 37th International Conference on Machine
  Learning}, 2020{\natexlab{b}}.

\end{thebibliography}
\bibliographystyle{icml2020}

\onecolumn
\newpage
\appendix


\section{Proofs}
\subsection{Proof of Theorem~\ref{thm:gen_dice_plus_convergence}}
\begin{proof}
We first rewrite the update for $d_t$ as
\begin{align}
d_{t+1} \doteq d_t + \alpha_t (G_{t+1} d_t + g_{t+1}) = d_t + \alpha_t (h(d_t) + M_{t+1}),
\end{align}
where
\begin{align}
h(d) \doteq g + Gd, \, M_{t+1} \doteq (G_{t+1} - G) d_t + (g_{t+1} - g)
\end{align}
The rest is the same as the proof of Theorem 2 in page 44 of \citet{maei2011gradient} without changing notations except for that we need to verify that the real parts of all eigenvalues of $G$ are negative.
Although our $G$ is more involved than the $G$ used in \citet{maei2011gradient},
a similar routine can be carried out to verify this condition.

We first show that the real part of any nonzero eigenvalue is strictly negative.
Let $\zeta \in \mathbb{C}, \zeta \neq 0$ be a nonzero eigenvalue of $G$ with normalized eigenvector $x$, i.e., $x^* x = 1$, where $x^*$ is the complex conjugate of $x$.
Hence $x^* Gx = \zeta, x \neq 0$. 
Let $x^\top = (x_1^\top, x_2^\top, x_3)$,
where $x_1 \in \mathbb{C}^{K}, x_2 \in \mathbb{C}^{K}, x_3 \in \mathbb{C}$, 
it is easy to verify 
\begin{align}
\zeta = -x_1^* C x_1 + x_2^* A^\top x_1 - x_1^*Ax_2 - \xi x_2^* x_2 + \lambda x_3^* d_\mu^\top X x_2 - \lambda x_2^* X^\top d_\mu x_3 - \lambda x_3^* x_3.
\end{align}
As $A$ is real, $A^\top = A^*$.
Consequently, 
$(x_2^* A^\top x_1)^* = x_1^*Ax_2$,
yielding
$\Re(x_2^* A^\top x_1 - x_1^*Ax_2) = 0$,
where $\Re(\cdot)$ denotes the real part.
Similarly, we can show
$\Re(\lambda x_3^* d_\mu^\top X x_2 - \lambda x_2^* X^\top d_\mu x_3) = 0$.
Consequently, we have
\begin{align}
\Re(\zeta) = \Re(x^* G x) = -x_1^*Cx_1 - \xi x_2^* x_2 - \lambda x_3^* x_3.
\end{align}
According to Assumption~\ref{assu:nonsingularity}, $x_1^* C x_1 \geq 0$,  
where the equality holds iff $x_1 = 0$.
When $\gamma < 1$, we have $\xi = 0$. 
Then $\zeta \neq 0$ implies at least one of $\{x_1, x_3\}$ is nonzero.
Consequently, we have $\Re(\zeta) < 0$.
When $\gamma = 1$, we have $\xi > 0$.
Then $\zeta \neq 0$ implies at least one of $\{x_1, x_2, x_3\}$ is nonzero.
Consequently, we have $\Re(\zeta) < 0$.

We then show 0 is not an eigenvalue of $G$, which completes the proof.
It suffices to show $\det(G) \neq 0$.
For a block matrix 
\begin{align}
Y = 
\begin{bmatrix}
Y_1 &Y_2 \\ Y_3 & Y_4
\end{bmatrix},
\end{align}
we have $\det(Y) = \det(Y_4)\det(Y_1 - Y_2 Y_4^{-1} Y_3) = \det(Y_1) \det(Y_4 - Y_3 Y_1^{-1} Y_2)$.
Applying this rule to $G$ yields
\begin{align}
\det(G) &= -\lambda \det (\begin{bmatrix}
-C & -A \\ A^\top & -\xi I
\end{bmatrix} + \lambda^{-1} \begin{bmatrix}
0 & 0 \\ 0 & -\lambda^2 X^\top d_\mu d_\mu^\top X
\end{bmatrix}) \\
&= (-1)^{2K+1} \lambda \det(\begin{bmatrix}
C & A \\ -A^\top & \xi I + \lambda X^\top d_\mu d_\mu^\top X
\end{bmatrix}) \\
&= (-1)^{2K + 1} \lambda \det(C) \det(\xi I + \lambda X^\top d_\mu d_\mu^\top X + A^\top C^{-1}A).
\end{align}
Note $\lambda X^\top d_\mu d_\mu^\top X$ is always positive semidefinite.
Assumption~\ref{assu:ridge} ensures at least one of $\{\xi I, A^\top C^{-1} A\}$ is strictly positive definite, 
which ensures $\det(G) \neq 0$.
\end{proof}

\subsection{Proof of Proposition~\ref{prop:opt_error}}
\begin{proof}
This proof employs results from Proposition 3.2 in \citet{nemirovski2009robust} by mapping Projected GradientDICE to the general mirror saddle-point stochastic approximation algorithm in \citet{nemirovski2009robust}. 
Given the similarity between Projected GradientDICE and Revised GTD \citep{liu2015finite}, 
most of our proof is a verbatim repetition (thus omitted) of the finite sample analysis for Revised GTD in the proof of Proposition 3 in \citet{liu2015finite} except for that we need to bound different terms.
Namely, the \emph{stochastic sub-gradient vector} for Projected GradientDICE is 
\begin{align}
G(w, y) \doteq \begin{bmatrix}
G_w(w, y) \\ -G_y(w, y)
\end{bmatrix} \doteq \begin{bmatrix}
G_{3, t}y + G_{4, t}w \\
-(G_{1, t}y + G_{2, t}w + G_{5, t})
\end{bmatrix}.
\end{align}
We need only to bound $\E[||G_w(w, y)||^2]$ and $\E[||G_y(w, y)||^2]$.
Assumption~\ref{assu:feasibility} ensures the following quantities are well-defined:
\begin{align}
D_W &\doteq (\max_{w \in W} ||w||^2 - \min_{w \in W} ||w||^2)^{\frac{1}{2}}, \\
D_Y &\doteq (\max_{y \in Y} ||y||^2 - \min_{y \in Y} ||y||^2)^{\frac{1}{2}}
\end{align}
We define shorthand $\bar{G}_i \doteq \E[G_{i, t}]$ for $i = 1, \dots 5$, all of which are constant (c.f. Eq~\eqref{eq:expected_update}).
Under Assumption~\ref{assu:feature}, it is easy to verify that there are constants $\sigma_i > 0$ such that
\begin{align}
\E[||G_{i,t} - \bar{G}_{i}]||^2] \leq \sigma_i^2.
\end{align}
Using the equality 
\begin{align}
\E[||x||^2] = \E[||x - \E[x] ||^2] + ||\E[x]||^2,
\end{align}
we have
\begin{align}
\E[||G_w(w, y)||^2] &\leq \E[||G_{3, t}||^2] D_Y^2 + \E[||G_{4, t}||^2] D_W^2 \\
&\leq \sigma_3^2 D_Y^2 + \sigma_3^2 ||\bar{G}_{3}||^2 + \xi^2 D_W^2 \\
&\doteq C_1 \\
\E[||G_y(w, y)||^2] &\leq \sigma_1^2 D_Y^2 ||\bar{G}_1||^2 + \sigma_2^2 D_W^2 ||\bar{G}_2||^2 + \sigma_5^2 ||\bar{G}_5||^2 \\
&\doteq C_2
\end{align}
We now have all gradients to invoke Proposition 3.2 in \citet{nemirovski2009robust}.
Particularly, the $M_*^2$ in Eq 3.5 in \citet{nemirovski2009robust} is now
\begin{align}
M_*^2 \doteq 2D_W^2C_1^2 + 2D_Y^2C_2^2.
\end{align}
According to Eq 3.12 in \citet{nemirovski2009robust}, 
we set the learning rates $\alpha_t$ in Projected GradientDICE as
\begin{align}
\alpha_t \doteq \frac{2c}{M_* \sqrt{n}},
\end{align}
where $c > 0$ is a constant.
Note this is a constant learning rate.
Now Proposition 3.2 in \citet{nemirovski2009robust} states for any $\Omega > 1$,
\begin{align}
\Pr\Big\{\eopt(\bar{w}_n, \bar{y}_n) > \sqrt{\frac{5}{n}}(8 + 2\Omega )\max\{c, c^{-1}\} M_* \Big\} \leq 2e^{-\Omega}.
\end{align}
For a small $\delta > 0$, setting $\Omega = -\ln \frac{\delta}{2}$ completes the proof.
\end{proof}

\section{Computational Details of the Hard Example for GenDICE}
In our instantiation, we have 
\begin{align}
J(\tau, f, \eta) &= \E_p[\tau(s, a)f(s^\prime, a^\prime)] - \E_{d_\mu}[\tau(s, a)(f(s, a) + \frac{1}{4}f(s, a)^2)] + \E_{d_\mu}[\eta \tau(s, a) - \eta] - \frac{\eta^2}{2}\\
&= \frac{1}{4}\tau_1^2 f_1 + \frac{1}{4}\tau_1^2 f_2
+ \frac{1}{4}\tau_2^2 f_1 + \frac{1}{4}\tau_2^2 f_2 
-\frac{1}{2} \tau_1^2 (f_1 + \frac{1}{4}f_1^2) - \frac{1}{2} \tau_2^2 (f_2 + \frac{1}{4}f_2^2) + \frac{1}{2} \eta \tau_1^2 + \frac{1}{2} \eta \tau_2^2 - \eta - \frac{1}{2}\eta^2, \\
\frac{\partial J}{\partial \tau_1} &= \frac{1}{2}\tau_1f_1 + \frac{1}{2}\tau_1f_2 - \tau_1(f_1 + \frac{1}{4}f_1^2) + \eta \tau_1, \\
\frac{\partial J}{\partial \tau_2} &= \frac{1}{2}\tau_2f_1 + \frac{1}{2}\tau_2 f_2 - \tau_2(f_2 + \frac{1}{4}f_2^2) + \eta \tau_2, \\
\frac{\partial J}{\partial f_1} &= \frac{1}{4}\tau_1^2 + \frac{1}{4}\tau_2^2  - \frac{1}{2}\tau_1^2(1 + \frac{1}{2}f_1), \\
\frac{\partial J}{\partial f_1} &= \frac{1}{4}\tau_1^2 + \frac{1}{4}\tau_2^2  - \frac{1}{2}\tau_1^2(1 + \frac{1}{2}f_2), \\
\frac{\partial J}{\partial \eta} &= \frac{1}{2} \tau_1^2 + \frac{1}{2} \tau_2^2 - 1 - \eta.
\end{align}
It is clear that the gradient vanishes at $(\tau_1, \tau_2, f_1, f_2, \eta) = (0, 0, 0, 0, -1)$.

\section{Details of Experiments}
We implemented \texttt{Boyan's Chain} \citep{boyan1999least} by ourselves.
\texttt{Reacher-v2} is from Open AI gym \citep{brockman2016openai}~\footnote{\url{https://gym.openai.com/}}, 
which is used as a benchmark in both \citet{nachum2019dualdice} and \citet{zhang2020gendice}.
For \texttt{Reacher-v2} experiments with neural networks, 
we use a batch size 128 for training.
For \texttt{Boyan's Chain} experiments,
the batch size is 1.
For DualDICE, the convex function to compose the loss is $f(x) = \frac{2}{3}|x|^{\frac{3}{2}}$ as recommended by \citet{nachum2019dualdice}.
We conducted our experiments in a server with an Intel\textsuperscript{\textregistered} Xeon\textsuperscript{\textregistered} Gold 6152 CPU.
Our implementation is based on PyTorch.

\subsection{State Features of Boyan's Chain}
We use exactly the same features as \citet{boyan1999least}.
\begin{align}
x(s_{12}) &\doteq [1, 0, 0, 0]^\top \\
x(s_{11}) &\doteq [0.75, 0.25, 0, 0]^\top \\
x(s_{10}) &\doteq [0.5, 0.5, 0, 0]^\top \\
x(s_9) &\doteq [0.25, 0.75, 0, 0]^\top \\
x(s_8) &\doteq [0, 1, 0, 0]^\top \\
x(s_7) &\doteq [0, 0.75, 0.25, 0]^\top \\
x(s_6) &\doteq [0, 0.5, 0.5, 0]^\top \\
x(s_5) &\doteq [0, 0.25, 0.75, 0]^\top \\
x(s_4) &\doteq [0, 0, 1, 0]^\top \\
x(s_3) &\doteq [0, 0, 0.75, 0.25]^\top \\
x(s_2) &\doteq [0, 0, 0.5, 0.5]^\top \\
x(s_1) &\doteq [0, 0, 0.25, 0.75]^\top \\
x(s_0) &\doteq [0, 0, 0, 1]^\top
\end{align}

\end{document}